\documentclass{article} 
\usepackage[preprint]{colm2026_conference}

\usepackage{microtype}
\usepackage{hyperref}
\usepackage{url}
\usepackage{booktabs}
\usepackage{amsmath}
\usepackage{amsthm}
\usepackage[ruled,vlined,linesnumbered]{algorithm2e}
\usepackage{graphicx}
\usepackage{subcaption}
\usepackage{multirow}
\theoremstyle{plain}
\newtheorem{theorem}{Theorem}[section]

\newtheorem{corollary}[theorem]{Corollary}
\theoremstyle{definition}

\newtheorem{assumption}[theorem]{Assumption}
\theoremstyle{remark}

\usepackage{lineno}

\definecolor{darkblue}{rgb}{0, 0, 0.5}
\hypersetup{colorlinks=true, citecolor=darkblue, linkcolor=darkblue, urlcolor=darkblue}

\title{Principled and Scalable Diversity-Aware Retrieval via Cardinality-Constrained Binary Quadratic Programming}

\author{Qiheng Lu \& Nicholas D. Sidiropoulos \\
University of Virginia\\
Charlottesville, VA 22904, USA \\
\texttt{\{luqh, nikos\}@virginia.edu} \\
}

\begin{document}

\ifcolmsubmission
\linenumbers
\fi

\maketitle

\begin{abstract}
Diversity-aware retrieval is essential for Retrieval-Augmented Generation (RAG), yet existing methods lack theoretical guarantees and face scalability issues as the number of retrieved passages $k$ increases. We propose a principled formulation of diversity retrieval as a cardinality-constrained binary quadratic programming (CCBQP), which explicitly balances relevance and semantic diversity through an interpretable trade-off parameter. Inspired by recent advances in combinatorial optimization, we develop a non-convex tight continuous relaxation and a Frank--Wolfe based algorithm with landscape analysis and convergence guarantees. Extensive experiments demonstrate that our method consistently dominates baselines on the relevance-diversity Pareto frontier, while achieving significant speedup.
\end{abstract}

\section{Introduction}

Retrieval-Augmented Generation (RAG) \citep{lewis2020retrieval} has emerged as a widely adopted paradigm for dynamically grounding large language models (LLMs) in external knowledge, reducing hallucinations and improving factual accuracy. However, standard top\nobreakdash-$k$ retrieval \citep{karpukhin2020dense} suffers from a fundamental limitation: retrieved passages tend to be semantically redundant, wasting the limited context window. Redundancy is pervasive in many real-world corpora. For example, news articles about the same event are frequently republished across outlets, resulting in near-duplicate passages that cluster tightly in the embedding space. Consequently, this clustering often leaves relevant passages on other aspects of the query unretrieved. Diversity-aware retrieval addresses this by explicitly penalizing redundancy, ensuring the retrieved passages collectively cover a broader range of relevant information.

Diversity-aware retrieval has been extensively studied in the information retrieval community, with Maximal Marginal Relevance (MMR) \citep{carbonell1998use} and Determinantal Point Processes (DPP) \citep{macchi1975coincidence,kulesza2012determinantal} being the most widely adopted approaches, often natively integrated into popular RAG frameworks (e.g., LlamaIndex, LangChain). However, both methods suffer from fundamental limitations.

MMR maximizes a non-monotone submodular function \citep{lin2011class}, for which the greedy algorithm provides no approximation guarantee in general, unlike the monotone submodular case \citep{nemhauser1978analysis}. Moreover, MMR requires $O (knd)$ time, where $k$ is the number of passages to retrieve, $n$ is the size of the candidate pool, and $d$ is the embedding dimension. This runtime scales linearly with $k$, becoming a practical bottleneck as the number of retrieved passages increases.

DPP offers an alternative probabilistic framework for diversity-aware retrieval, but it suffers from two fundamental limitations. First, from a computational perspective, finding the maximum a posteriori (MAP) configuration --- the most probable subset --- is NP-hard \citep{kulesza2012determinantal}. \citet{gillenwater2012near} propose an algorithm with a 1/4-approximation guarantee for the unconstrained DPP MAP inference problem under the assumption that the log-determinant objective is non-negative, a condition that may not hold in practice. Its quartic time complexity makes it impractical for large-scale retrieval and the theoretical guarantee does not directly extend to the cardinality-constrained setting. \citet{chen2018fast} propose a more efficient greedy algorithm with $O(knd + k^{2} n)$ complexity via Cholesky decomposition, but this comes at the cost of losing any approximation guarantee. The runtime scales super-linearly with $k$, presenting an even more severe scalability challenge than MMR. Second, from a formulation perspective, the relevance-diversity trade-off parameter in DPP is difficult to interpret and tune in practice. This is because the relevance term is naturally bounded, whereas the log-determinant diversity term is unbounded below.

As LLMs increasingly support longer context windows, the demand for retrieving a larger number of diverse passages (i.e., increasing $k$) further exacerbates these efficiency issues. In this paper, we propose a principled and scalable approach for diversity retrieval in RAG by formulating it as a cardinality-constrained binary quadratic programming (CCBQP), where the objective explicitly balances relevance and semantic diversity through a trade-off parameter $\theta$.

While CCBQP is NP-hard \citep{feige2001dense} and difficult to approximate \citep{khot2006ruling,manurangsi2017almost} in general, building on recent progresses \citep{lu2025densest,lu2026densest} in graph combinatorial optimization, we propose a non-convex tight continuous relaxation which can be efficiently tackled via the Frank--Wolfe algorithm \citep{frank1956algorithm,jaggi2013revisiting,braun2022conditional}. We further provide a landscape analysis of the relaxed objective and establish convergence guarantees for the proposed algorithm. Extensive experiments on QAMPARI \citep{amouyal2023qampari} and ASQA \citep{stelmakh2022asqa} with retrieval sizes $k \in \{ 25, 50, 100 \}$ demonstrate that our method consistently dominates MMR and DPP on the relevance-diversity Pareto frontier. In the practically relevant regime of the relevance-diversity trade-off parameter $\theta \geq 0.5$, our method achieves $2.4 \times$ to $22.9 \times$ speedup over MMR (DPP is slower than MMR), with speedup increasing as $\theta$ and $k$ grow --- as our algorithm's runtime scales sub-linearly in $k$, in contrast to the linear scaling of MMR and the super-linear scaling of DPP in $k$. End-to-end experiments further show that diversity-aware retrieval yields consistent albeit modest improvements in generation quality over top-$k$ retrieval.

Our main contributions are as follows:

\noindent $\bullet$ \textbf{Problem Formulation:} We propose a principled formulation of diversity-aware retrieval in RAG as a cardinality-constrained binary quadratic program (CCBQP), which explicitly balances relevance and semantic diversity through an interpretable trade-off parameter.

\noindent $\bullet$ \textbf{Theoretical Analysis:} We propose a non-convex tight continuous relaxation which we efficiently tackle via Frank--Wolfe. We further provide a landscape analysis of the relaxed objective and establish convergence guarantees for the proposed algorithm.

\noindent $\bullet$ \textbf{Empirical Evaluation:} Extensive experiments on QAMPARI and ASQA with different retrieval sizes demonstrate that our method consistently dominates MMR and DPP on the relevance-diversity Pareto frontier while achieving up to $22.9 \times$ speedup over MMR. End-to-end RAG experiments further confirm the practical value of diversity-aware retrieval.

\section{Problem Formulation}

Let $\boldsymbol{E} \in \mathbb{R}^{n \times d}$ denote the embedding matrix of the candidate passage pool, where each row $e_{i} \in \mathbb{R}^{d}$ is a normalized embedding of the $i$-th passage, $n$ is the size of the candidate pool, and $d$ is the embedding dimension. Let $q \in \mathbb{R}^{d}$ be the normalized query embedding, and let $c = Eq \in \mathbb{R}^{n}$ denote the relevance score vector, where $c_{i} = e_{i}^{\mathsf{T}} q$ measures the cosine similarity between the $i$-th passage and the query. The goal of diversity-aware retrieval is to select a subset of $k$ passages that collectively balances relevance to the query and semantic diversity among the selected passages.

Standard top-$k$ retrieval selects passages by solving
\begin{equation}
     \max_{\boldsymbol{x} \in \{0,1\}^n} \boldsymbol{c}^{\mathsf{T}} \boldsymbol{x} \quad \text{s.t.} \quad \boldsymbol{1}^{\mathsf{T}} \boldsymbol{x} = k,
\end{equation}
which maximizes the total relevance of the selected passages but ignores semantic redundancy. To promote diversity, we introduce the term $\boldsymbol{x}^{\mathsf{T}} (\boldsymbol{I} - \boldsymbol{EE}^{\mathsf{T}}) \boldsymbol{x}$, which penalizes the total pairwise cosine similarity among the selected passages, thereby encouraging the retrieved set to be semantically diverse.

To balance relevance and diversity, we propose the following CCBQP formulation:
\begin{equation}
    \begin{aligned}
        \max_{\boldsymbol{x} \in \{0,1\}^{n}}& \theta \cdot (k - 1) \cdot \boldsymbol{c}^{\mathsf{T}} \boldsymbol{x} + (1 - \theta) \boldsymbol{x}^{\mathsf{T}} (\boldsymbol{I} - \boldsymbol{EE}^{\mathsf{T}}) \boldsymbol{x} \\
        \text{s.t.}& \quad \boldsymbol{1}^{\mathsf{T}} \boldsymbol{x} = k, \label{p1}
    \end{aligned}
\end{equation}
where $\theta \in [0, 1]$ is a trade-off parameter. The coefficient $k - 1$ is introduced for the alignment of the scale: both the relevance term $(k - 1) \boldsymbol{c}^{\mathsf{T}} \boldsymbol{x}$ and the diversity term $\boldsymbol{x}^{\mathsf{T}} (\boldsymbol{I} - \boldsymbol{EE}^{\mathsf{T}}) \boldsymbol{x}$ are bounded in $[-(k - 1) k, (k - 1)k]$, ensuring that the trade-off parameter $\theta$ has a consistent interpretation independent of $k$.

Inspired by \citet{lu2025densest, lu2026densest}, we first reformulate \eqref{p1} on the binary feasible set, where $\Vert \boldsymbol{x} \Vert_{2}^{2} = k$ is constant, so adding $(1 - \theta) \lambda \Vert \boldsymbol{x} \Vert_{2}^{2}$ to the objective yields an equivalent problem
\begin{equation}
    \begin{aligned}
        \max_{\boldsymbol{x} \in \{0,1\}^{n}}& f(\boldsymbol{x}) = \theta \cdot (k - 1) \cdot \boldsymbol{c}^{\mathsf{T}} \boldsymbol{x} + (1 - \theta) \boldsymbol{x}^{\mathsf{T}} (\lambda \boldsymbol{I} - \boldsymbol{EE}^{\mathsf{T}}) \boldsymbol{x} \\
        \text{s.t.}& \quad \boldsymbol{1}^{\mathsf{T}} \boldsymbol{x} = k. \label{p2}
    \end{aligned}
\end{equation}

Relaxing the binary constraint in \eqref{p2} to its convex hull then gives
\begin{equation}
    \begin{aligned}
        \max_{\boldsymbol{x} \in [0,1]^{n}}& f(\boldsymbol{x}) = \theta \cdot (k - 1) \cdot \boldsymbol{c}^{\mathsf{T}} \boldsymbol{x} + (1 - \theta) \boldsymbol{x}^{\mathsf{T}} (\lambda \boldsymbol{I} - \boldsymbol{EE}^{\mathsf{T}}) \boldsymbol{x} \\
        \text{s.t.}& \quad \boldsymbol{1}^{\mathsf{T}} \boldsymbol{x} = k. \label{p3}
    \end{aligned}
\end{equation}

We say that the relaxation from  \eqref{p2} to \eqref{p3} is \emph{tight} if their global maximum values coincide, or equivalently, if all global maximizers of \eqref{p2} remain optimal for \eqref{p3}.

\begin{theorem} \label{thm1}
    If $\lambda \geq 2$, the relaxation from \eqref{p2} to \eqref{p3} is tight for any $k$. 
\end{theorem}

\begin{proof}
    Please refer to Appendix \ref{appx1}.
\end{proof}

While Theorem \ref{thm1} guarantees the existence of an integral global maximizer for any $\lambda \geq 2$, it does not preclude the existence of non-integral global maximizers when $\lambda = 2$. To rule out non-integral solutions, we introduce a realistic assumption on the passage embeddings.

\begin{assumption} \label{assump1}
    For any pair of distinct passages in the candidate pool, their $\ell^{2}$-normalized embeddings are not perfectly anti-aligned. That is, the cosine similarity satisfies $\boldsymbol{e}_{i}^{\mathsf{T}} \boldsymbol{e}_{j} > -1$.
\end{assumption}

This assumption holds in practice. It is a well-documented phenomenon that representations generated by modern dense embedders are highly anisotropic, typically occupying a narrow cone within the embedding space \citep{ethayarajh2019contextual,li2020sentence}. Under this geometric premise, perfectly opposite vectors virtually never exist in real-world corpora.

Under this assumption, by leveraging the exact same construction used in the proof of Theorem \ref{thm1}, one can prove that for any non-integral point, there exists a strict ascent direction when $\lambda \geq 2$. This precludes any non-integral point  from being a global maximizer, which naturally yields the following corollary:

\begin{corollary} \label{cor1}
    Under Assumption \ref{assump1}, setting $\lambda \geq 2$ guarantees that the sets of optimal solutions for \eqref{p2} and \eqref{p3} are identical, i.e., all global maximizers of the relaxed problem \eqref{p3} are integral.
\end{corollary}

Beyond establishing the tightness of the relaxation, it is crucial to understand the optimization landscape of the relaxed problem \eqref{p3}, as it inherently dictates the convergence behavior of gradient-based algorithms.

\begin{theorem}[Strict Dichotomy of Stationary Points] \label{thm2}
    Under Assumption \ref{assump1}, for any penalty parameter $\lambda \geq 2$ and trade-off parameter $\theta \in [0, 1)$, every stationary point of the relaxed problem \eqref{p3} is either a local maximizer or a strict saddle point with explicit positive curvature. Furthermore, every local maximizer is integral.
\end{theorem}

\begin{proof}
    Please refer to Appendix \ref{appx2}.
\end{proof}

Theorem \ref{thm2} reveals a benign optimization landscape for the relaxed formulation. Because all stationary points that are not local maximizers are strict saddles with explicit positive curvature, standard gradient-based optimization algorithms can typically escape these unstable regions without requiring complex saddle-escaping mechanisms. More importantly, since every local maximizer is integral, any algorithm that converges to a local maximizer is guaranteed to yield a discrete subset, which entirely eliminates the need for rounding.

\begin{theorem}[Monotonicity of Local Maximizers] \label{thm3}
    Let $f_{\lambda}(\boldsymbol{x})$ denote the objective function of the relaxed problem \eqref{p3} parameterized by $\lambda$. Under Assumption \ref{assump1}, suppose that $\lambda_{2} > \lambda_{1} \geq 2$. If an integral point $\boldsymbol{x}^{\ast}$ is a local maximizer of $f_{\lambda_{1}}(\boldsymbol{x})$, then it is also a local maximizer of $f_{\lambda_{2}}(\boldsymbol{x})$.
\end{theorem}

\begin{proof}
    Please refer to Appendix \ref{appx3}.
\end{proof}

Theorem \ref{thm3} reveals that as $\lambda$ increases, the set of local maximizers monotonically expands, introducing spurious local maximizers into the optimization landscape. Consequently, to minimize the proliferation of such spurious local maximizers while guaranteeing the tightness of the relaxation, we set $\lambda = 2$ in the remainder of this paper.

\section{Optimization Algorithm}

Having established the favorable properties of the optimization landscape of the relaxed problem \eqref{p3}, we now introduce an efficient optimization algorithm to tackle it. We employ the Frank--Wolfe algorithm \citep{frank1956algorithm,jaggi2013revisiting,braun2022conditional}, which is particularly well-suited for problems with a polyhedral feasible set, such as \eqref{p3}.

\begin{algorithm}[tb]
    \caption{Frank-Wolfe with Exact Line Search for \eqref{p3}}
    \label{alg:fw}
    \SetKwInOut{Input}{Input}
    \SetKwInOut{Output}{Output}
    \DontPrintSemicolon

    \Input{Document embeddings $\boldsymbol{E} \in \mathbb{R}^{n \times d}$, query relevance $\boldsymbol{c} \in \mathbb{R}^n$, budget $k$, trade-off parameter $\theta$, penalty parameter $\lambda=2$, and maximum \# iterations $T$.}
    \Output{Document indicator vector $\boldsymbol{x} \in \mathbb{R}^{n}$.}

    Initialize $\boldsymbol{x}^{(0)} \in \{ \boldsymbol{x} \in [0, 1]^{n} \mid \boldsymbol{1}^{\mathsf{T}} \boldsymbol{x} = k \}$ \;
    Initialize $\boldsymbol{v}^{(0)} \leftarrow \boldsymbol{E}^{\mathsf{T}} \boldsymbol{x}^{(0)}$\;
    
    \For{$t = 0, 1, \dots, T-1$}{
        $\nabla f(\boldsymbol{x}^{(t)}) \leftarrow \theta (k-1)\boldsymbol{c} + 2(1-\theta)(\lambda \boldsymbol{x}^{(t)} - \boldsymbol{E}\boldsymbol{v}^{(t)})$ \tcp*{Gradient Evaluation}
        $\boldsymbol{s}^{(t)} \leftarrow \operatorname{top-k}(\nabla f(\boldsymbol{x}^{(t)}))$ \tcp*{Linear Maximization Oracle}
        $\boldsymbol{d}^{(t)} \leftarrow \boldsymbol{s}^{(t)} - \boldsymbol{x}^{(t)}$\;
        
        \If{$\langle \nabla f(\boldsymbol{x}^{(t)}), \boldsymbol{d}^{(t)} \rangle = 0$}{
            \Return $\boldsymbol{x}^{(t)}$ \tcp*{Exact Convergence}
        }
        $\gamma^{(t)} \leftarrow \arg\max_{\gamma \in [0, 1]} f(\boldsymbol{x}^{(t)} + \gamma \boldsymbol{d}^{(t)})$ \tcp*{Exact Line Search}
        $\boldsymbol{x}^{(t+1)} \leftarrow \boldsymbol{x}^{(t)} + \gamma^{(t)} \boldsymbol{d}^{(t)}$\;
        $\boldsymbol{v}^{(t+1)} \leftarrow \boldsymbol{v}^{(t)} + \gamma^{(t)} \boldsymbol{E}^{\mathsf{T}}\boldsymbol{d}^{(t)}$\;
    }
    \Return $\boldsymbol{x}^{(t+1)}$
\end{algorithm}

Algorithm \ref{alg:fw} delineates the step-by-step execution of our proposed method. While grounded in the standard non-convex Frank--Wolfe framework, our algorithmic implementation is specifically optimized for diversity-aware retrieval in RAG.

\textbf{Exact Line Search.} While widely used adaptive step sizes in non-convex Frank--Wolfe depend on the Lipschitz constant $L$ \citep{lacoste2016convergence,bertsekas2016nonlinear}, this becomes a severe bottleneck for RAG. It is a well-documented phenomenon that dense embedding vectors are highly clustered in a narrow cone \citep{ethayarajh2019contextual,li2020sentence}. Consequently, the maximum eigenvalue of the Gram matrix $\boldsymbol{EE}^{\mathsf{T}}$ scales almost linearly with the candidate pool size $n$. For large $n$, the Lipschitz constant $L$ becomes catastrophically large, making the convergence rate of these adaptive step sizes extremely slow.

We bypass this bottleneck by using an exact line search (Line 9). Since the objective $f(\boldsymbol{x})$ is quadratic, the function value along the update direction $\boldsymbol{d}^{(t)}$ can be formulated as a quadratic function of the step size $\gamma$:
\begin{equation}
    f(\boldsymbol{x}^{(t)} + \gamma \boldsymbol{d}^{(t)}) = f(\boldsymbol{x}^{(t)}) + \gamma \Delta + \frac{1}{2}\gamma^{2} C,
\end{equation}
where $\Delta = \langle \nabla f(\boldsymbol{x}^{(t)}), \boldsymbol{d}^{(t)} \rangle$ is the non-negative Frank--Wolfe gap and $C = \boldsymbol{d}^{(t)\mathsf{T}} \nabla^2 f(\boldsymbol{x}^{(t)}) \boldsymbol{d}^{(t)}$ represents the local curvature. Substituting the Hessian, the curvature expands to $C = 2(1 - \theta) ( \lambda \Vert \boldsymbol{d}^{(t)} \Vert_{2}^{2} - \Vert\boldsymbol{E}^{\mathsf{T}}\boldsymbol{d}^{(t)}\Vert_{2}^{2} )$. 

The optimal step size $\gamma^{(t)}$ maximizing this parabola over $[0, 1]$ has a closed-form solution:
\begin{equation}
    \gamma^{(t)} = \begin{cases} 
    1, & \text{if } C \geq 0, \\ 
    \min\left(1, \frac{\Delta}{-C}\right), & \text{if } C < 0.
    \end{cases}
\end{equation}
This exact line search step size significantly accelerates convergence compared to Lipschitz-based adaptive step sizes.

\textbf{Accelerating General Matrix-Vector Multiplication via Sparsity.} Both the exact line search (computing curvature $C$) and updating $\boldsymbol{v}^{(t+1)}$ (Line 11) depend on the computation of $\boldsymbol{E}^{\mathsf{T}} \boldsymbol{d}^{(t)}$. Naively, this constitutes a general matrix-vector multiplication (GEMV) that requires $O(nd)$ operations. Furthermore, operating on a large, dense transposed matrix $\boldsymbol{E}^{\mathsf{T}}$ usually exhibits poor cache locality, limiting the actual execution speed.

We circumvent this bottleneck by decomposing $\boldsymbol{E}^{\mathsf{T}}\boldsymbol{d}^{(t)}$ as $\boldsymbol{E}^{\mathsf{T}}\boldsymbol{s}^{(t)} - \boldsymbol{E}^{\mathsf{T}}\boldsymbol{x}^{(t)}$. The second term is $\boldsymbol{v}^{(t)}$, which is known and requires no recomputation. For the first term, since $\boldsymbol{s}^{(t)}$ is a $k$-sparse indicator vector, evaluating $\boldsymbol{E}^{\mathsf{T}}\boldsymbol{s}^{(t)}$ reduces to simply gathering and summing $k$ specific rows of $\boldsymbol{E}$. This drops the computation from $O(nd)$ to $O(kd)$ and is highly cache-friendly, as it only fetches the necessary $k$ rows.

\textbf{Complexity.} The execution of Algorithm \ref{alg:fw} is dominated by three main stages. First, the gradient evaluation (Line 4) requires a GEMV $\boldsymbol{E}\boldsymbol{v}^{(t)}$, taking $O(nd)$ time. Second, the linear maximization oracle (Line 5) identifies the top-$k$ entries of the gradient, which takes $O(n)$ time by using radix selection \citep{alabi2012fast} on fixed bit-width floating-point numbers (e.g., \texttt{torch.topk}) or $O(n)$ time in the worst case by using introselect \citep{musser1997introspective} (e.g., \texttt{numpy.argpartition}). Third, as discussed above, the computation of $\boldsymbol{E}^{\mathsf{T}} \boldsymbol{d}^{(t)}$ takes $O(kd)$ time. 

Taking these stages together, the total per-iteration time complexity is $O(nd)$, which is independent of $k$. In practice, because the computation of $\boldsymbol{E}^{\mathsf{T}} \boldsymbol{d}^{(t)}$ takes $O(kd)$, the actual wall-clock execution time per iteration exhibits a highly favorable, sublinear scaling in $k$.

\textbf{Convergence Analysis.} In general, the non-convex Frank--Wolfe algorithm with exact line search has an asymptotic convergence rate of $\mathcal{O}(1/\sqrt{t})$ to a stationary point \citep{lacoste2016convergence}. However, benefiting from the benign optimization landscape of the relaxed problem \eqref{p3} and the exact line search strategy for the step size selection, our algorithm enjoys  \emph{local exact convergence}.

\begin{theorem}[Local Exact Convergence] \label{thm4}
    Let $\boldsymbol{x}^{\ast}$ be an integral local maximizer of the relaxed problem \eqref{p3}. There exists a neighborhood $\mathcal{N}(\boldsymbol{x}^{\ast})$ around $\boldsymbol{x}^{\ast}$ such that, if any non-stationary point $\boldsymbol{x}^{(t)}$ generated by Algorithm \ref{alg:fw} enters $\mathcal{N}(\boldsymbol{x}^{\ast})$, the algorithm exactly converges to $\boldsymbol{x}^{\ast}$ in the next iteration.
\end{theorem}

\begin{proof}
    Please refer to Appendix \ref{appx4}.
\end{proof}

Although Theorem \ref{thm4} only provides a local guarantee, its practical implications are essentially global in practice. As proven in Theorem \ref{thm2}, all stationary points that are not local maximizers are strict saddle points with explicit positive curvature. Consequently, the algorithm almost never stalls around saddle points, practically ensuring global exact convergence to a local maximizer. 

\section{Experiments}

To empirically validate the effectiveness and efficiency of our proposed algorithm, we design comprehensive experiments to answer the following research questions:

\noindent $\bullet$ \textbf{RQ1 (Retrieval Efficiency):} Is our approach more computation-efficient than existing diversity-aware methods for large-scale retrieval? 

\noindent $\bullet$ \textbf{RQ2 (Retrieval Quality):} Does our approach achieve a better Pareto frontier on the relevance-diversity trade-off than existing methods?

\noindent $\bullet$ \textbf{RQ3 (Generation Quality):} Does the improved diversity in the retrieved passages enhance the generation quality?

\subsection{Experimental Setup}

\textbf{Datasets.} We evaluate our proposed method on two knowledge-intensive datasets: \textbf{ASQA} \citep{stelmakh2022asqa}, which focuses on answering ambiguous queries requiring multi-faceted explanations, and \textbf{QAMPARI} \citep{amouyal2023qampari}, which targets list-based questions requiring the retrieval of multiple distinct entities.

\textbf{Candidate Pools.} To construct the candidate pools, we first use BM25 sparse retrieval (via Pyserini \citep{lin2021pyserini}) to fetch the top-3000 passages from a corpus of approximately 21 million distinct 100-word Wikipedia passages for each query in the development set. We then merge and deduplicate these retrieved passages across all queries in the same dataset. This process yields a unified candidate pool of $n = 2,253,350$ passages for ASQA and $n = 2,064,511$ passages for QAMPARI. Next, we compute the 1024-dimensional dense embeddings using the \texttt{BGE-M3} model \citep{chen2024m3}. During evaluation, for each query, algorithms operate directly on this entire pool of size $n$ to select a subset of size $k$.

\textbf{Baselines.} We compare our proposed algorithm against two widely adopted diversity-aware retrieval methods:

\noindent $\bullet$ \textbf{Maximal Marginal Relevance (MMR):} A greedy algorithm that sequentially selects passages by explicitly trading off the relevance to the query against the maximum similarity to the already selected passages.

\noindent $\bullet$ \textbf{Determinantal Point Processes (DPP):} A probabilistic framework that finds the most probable subset. Because the exact DPP MAP inference is NP-hard, we employ the greedy approximation proposed by \citet{chen2018fast}.

\textbf{Implementation Details.} All retrieval experiments are conducted on a workstation equipped with an AMD Ryzen Threadripper PRO 5975WX CPU and two NVIDIA GeForce RTX 4090 GPUs. To ensure performance consistency, all retrieval experiments are executed on the non-display GPU. All retrieval algorithms are implemented in Python with PyTorch and available at \url{https://github.com/luqh357/CCBQP-RAG}, where all baselines are carefully optimized to ensure fair comparison. The generation experiments are performed on a computing cluster using a single NVIDIA A100 (80GB) Tensor Core GPU.

\subsection{Retrieval Experiments}

\textbf{Evaluation Metrics.} 
Unlike traditional information retrieval tasks that prioritize early precision using rank-aware metrics such as mean reciprocal rank (MRR) \citep{voorhees2000trec} or normalized discounted cumulative gain (nDCG) \citep{jarvelin2002cumulated}, we argue that these metrics are mismatched with RAG for two primary reasons. 

First, in typical RAG pipelines, retrieved passages are subsequently processed by a reranker, rendering the initial output ordering irrelevant at the retrieval stage. Second, even without a reranker, monotonically decaying metrics like MRR and nDCG are unsuited for RAG, as LLMs exhibit U-shaped attention over long contexts, which is known as the ``lost in the middle'' phenomenon \citep{liu2024lost}.

Therefore, we utilize \textbf{Recall@$k$} as the metric for relevance, which measures the proportion of required ground-truth evidence successfully covered within the entire retrieved subset. To quantify subset diversity, we employ the \textbf{Intra-List Average Distance (ILAD)} \citep{zhang2008avoiding}, calculated as the mean pairwise cosine distance among the embeddings of the selected $k$ passages. A higher ILAD indicates a more semantically diverse subset.

\begin{figure*}[t]
    \centering
    \begin{subfigure}[b]{0.497\textwidth}
        \includegraphics[width=\textwidth]{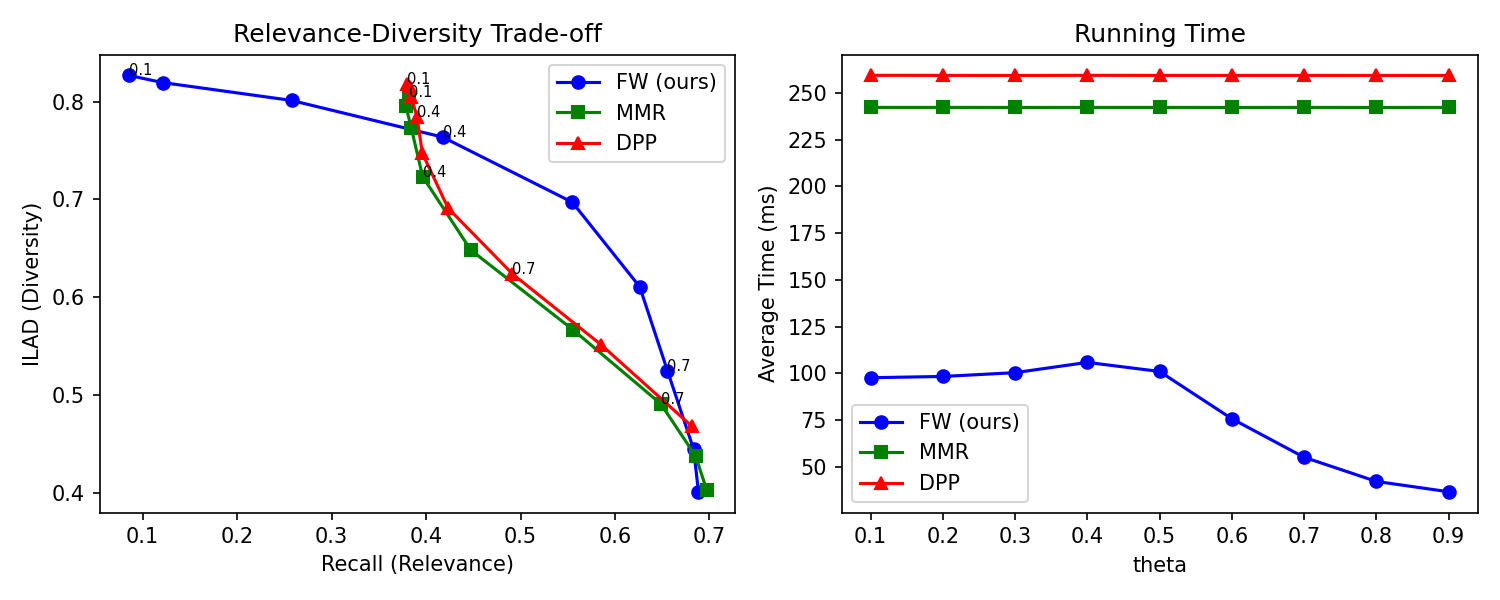}
        \caption{ASQA ($k=25$)}
        \label{fig:asqa_25}
    \end{subfigure}
    \hfill
    \begin{subfigure}[b]{0.497\textwidth}
        \includegraphics[width=\textwidth]{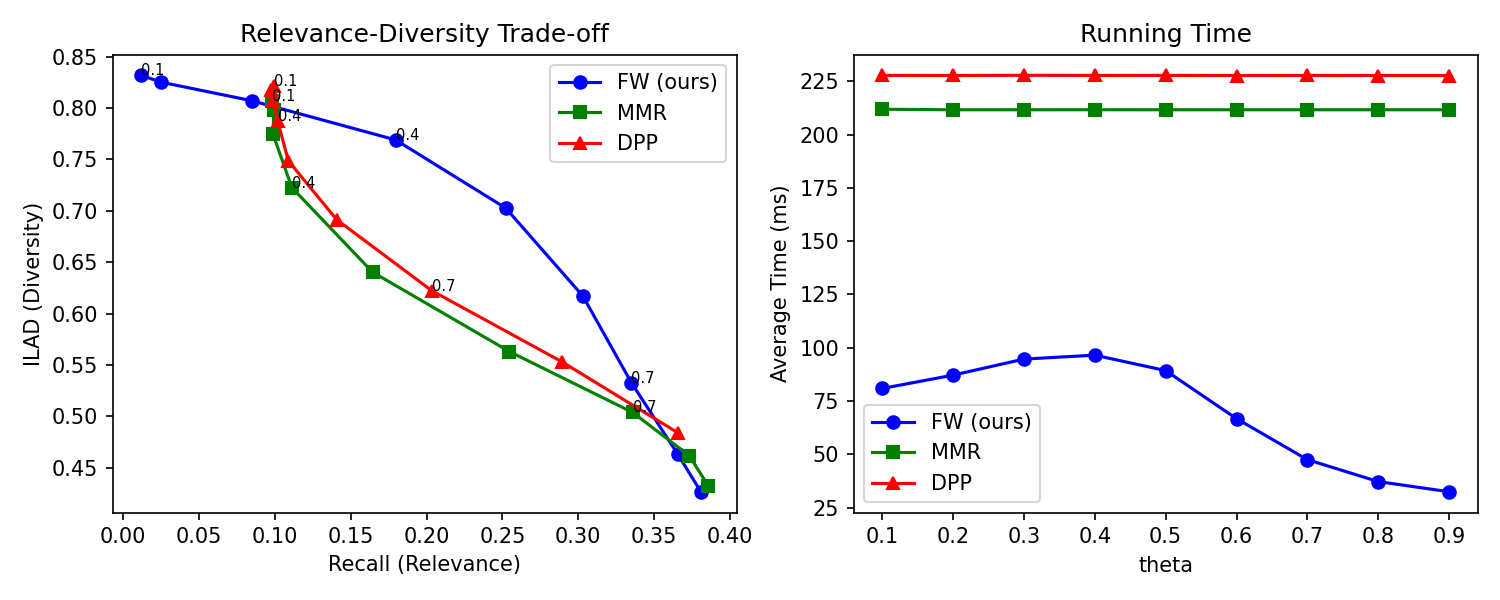}
        \caption{QAMPARI ($k=25$)}
        \label{fig:qampari_25}
    \end{subfigure}
    
    \vspace{0.4cm} 
    
    \begin{subfigure}[b]{0.497\textwidth}
        \includegraphics[width=\textwidth]{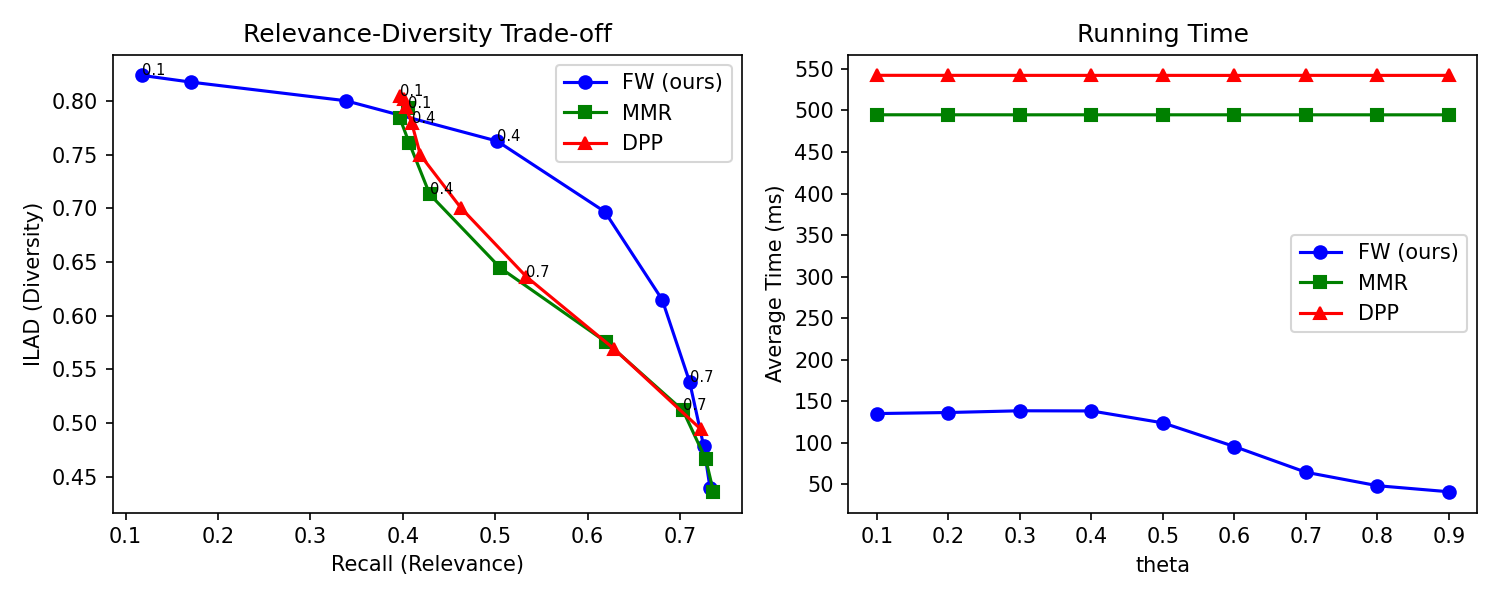}
        \caption{ASQA ($k=50$)}
        \label{fig:asqa_50}
    \end{subfigure}
    \hfill
    \begin{subfigure}[b]{0.497\textwidth}
        \includegraphics[width=\textwidth]{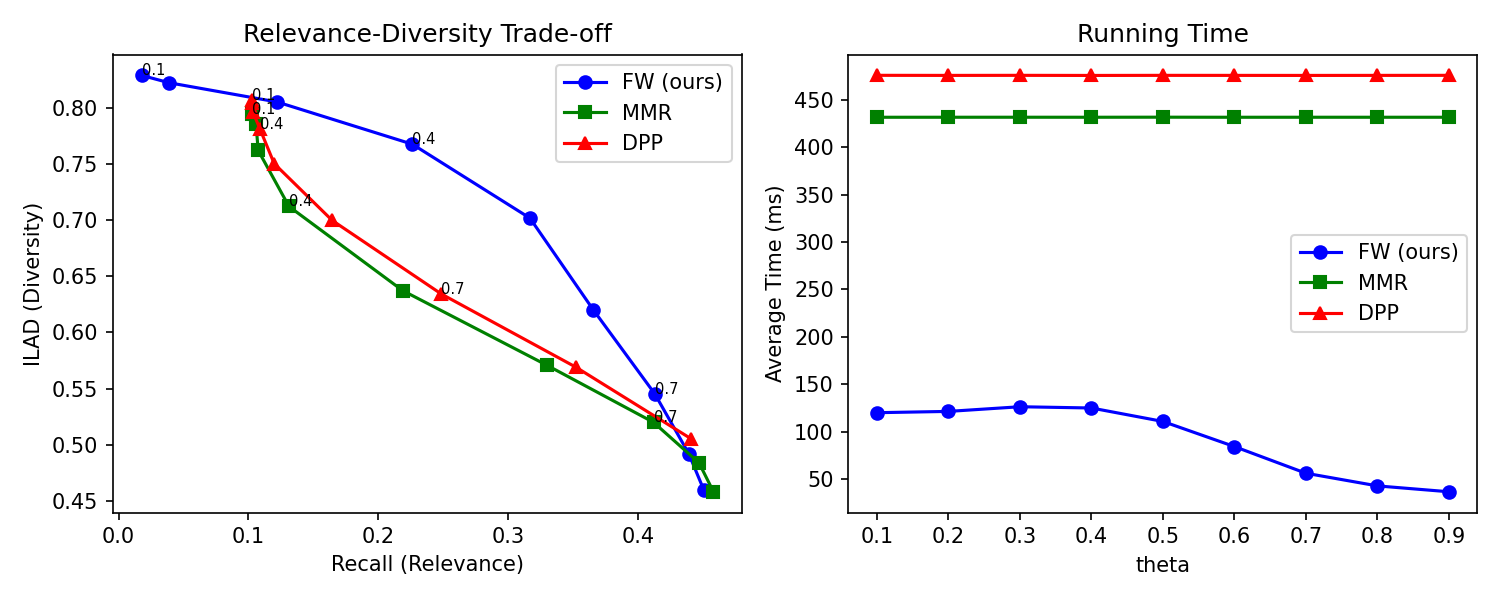}
        \caption{QAMPARI ($k=50$)}
        \label{fig:qampari_50}
    \end{subfigure}
    
    \vspace{0.4cm}
    
    \begin{subfigure}[b]{0.497\textwidth}
        \includegraphics[width=\textwidth]{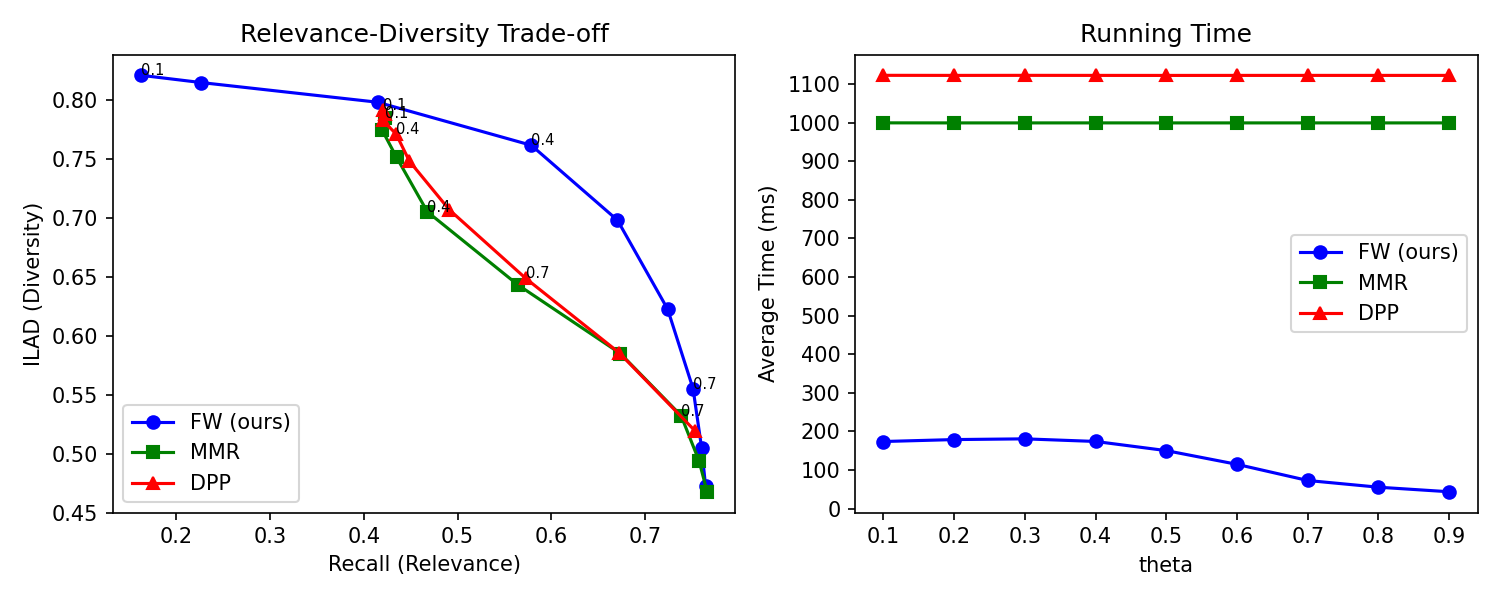}
        \caption{ASQA ($k=100$)}
        \label{fig:asqa_100}
    \end{subfigure}
    \hfill
    \begin{subfigure}[b]{0.497\textwidth}
        \includegraphics[width=\textwidth]{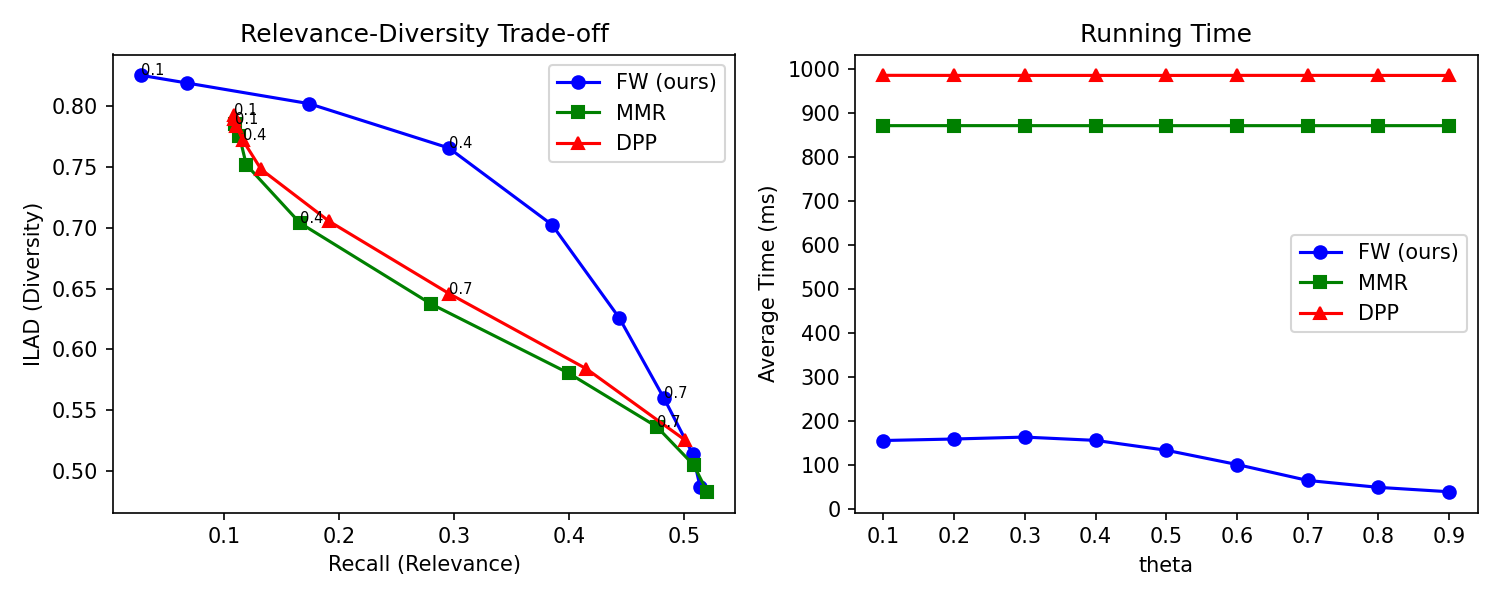}
        \caption{QAMPARI ($k=100$)}
        \label{fig:qampari_100}
    \end{subfigure}
    
    \caption{\textbf{Relevance-diversity trade-off and efficiency.} Relevance-diversity trade-off and computational efficiency on ASQA and QAMPARI across $k \in \{25, 50, 100\}$. The left panels show the Pareto frontier of Recall vs. ILAD by modulating the trade-off parameter $\theta \in [0.1, 0.9]$. The right panels report the per-query wall-clock latency (ms). Our algorithm consistently yields superior subset quality and significant speedup.}
    \label{fig:retrieval_main}
\end{figure*}

\textbf{RQ1 \& RQ2: Retrieval Efficiency \& Quality.}
To comprehensively evaluate the retrieval performance, we modulate the trade-off parameter $\theta \in [0.1, 0.9]$ and plot the relevance-diversity trade-off alongside the per-query wall-clock latency. Figure \ref{fig:retrieval_main} presents the complete evaluation results for subset sizes $k \in \{25, 50, 100\}$ across both the ASQA and QAMPARI datasets.

Observing the left panels across all settings, our Frank--Wolfe algorithm consistently establishes a superior Pareto frontier over both MMR and DPP. Unlike greedy baselines that prematurely sacrifice overall recall to satisfy local diversity constraints, our algorithm globally optimizes the relevance-diversity trade-off simultaneously.

Furthermore, the right panels demonstrate the efficiency on large-scale candidate pools ($n > 2 \times 10^6$). Perfectly aligning with our theoretical analysis, the empirical running times of MMR and DPP scale linearly and slightly super-linearly with $k$, respectively. Consequently, for $k = 100$, they require more than $850$ ms, which is prohibitive for real-time RAG. In contrast, our algorithm exhibits an overall sub-linear scaling, reducing this latency to merely $40 \sim 180$ ms. This advantage makes our approach well-suited for RAG pipelines that demand large context sizes.

It is worth noting that our algorithm yields lower recalls than the baselines for small $\theta$ values. This can be understood by noting that greedy algorithms like MMR select their very first passage based solely on relevance, thereby preventing the minimum recall from dropping too low. Additionally, the DPP data points are observed to tightly cluster together in this region, which may be attributed to the scale mismatch issue discussed earlier. However, such choices of $\theta$ are not practical for RAG, which requires sufficient relevance (typically $\theta \geq 0.5$) to prevent hallucination.

Within the practical operating regime (moderate to large $\theta$), the advantages of our algorithm are noticeable. For moderate $\theta$ values (e.g., $0.5 \sim 0.7$), our approach dominates the relevance-diversity trade-off while achieving considerable speedup. For large $\theta$ values (e.g., $0.8 \sim 0.9$), all algorithms converge toward standard top-$k$ retrieval, causing these data points to cluster together. However, our algorithm is significantly faster in this regime, taking only approximately 40 ms even when $k = 100$.

\subsection{Generation Experiments}

\textbf{RQ3: Generation Quality.} 
To evaluate the downstream generation quality, we employ \texttt{Llama-3.1-8B-Instruct} as our reader LLM. All generation experiments are using vLLM \citep{kwon2023efficient}. We set the generation temperature to 0.

\begin{table*}[t]
\centering
\resizebox{\textwidth}{!}{
\begin{tabular}{@{}ll | ccccc | ccccc@{}}
\toprule
\multirow{2}{*}{\textbf{Method}} & \multirow{2}{*}{$\boldsymbol{\theta}$} & \multicolumn{5}{c|}{$\boldsymbol{k = 25}$} & \multicolumn{5}{c}{$\boldsymbol{k = 50}$} \\
\cmidrule(lr){3-7} \cmidrule(l){8-12}
& & P & R & F1 & R$\geq$0.8 & F1$\geq$0.5 & P & R & F1 & R$\geq$0.8 & F1$\geq$0.5 \\
\midrule
Top-$k$ & --- & 0.2196 & 0.2134 & 0.1824 & 0.0600 & 0.1140 & \textbf{0.2032} & 0.2423 & 0.1847 & 0.0770 & \textbf{0.1200} \\
\midrule
\multirow{4}{*}{MMR}
 & 0.6 & 0.1777 & 0.1481 & 0.1327 & 0.0340 & 0.0570 & 0.1618 & 0.1692 & 0.1356 & 0.0430 & 0.0660 \\
 & 0.7 & 0.1913 & 0.1899 & 0.1579 & 0.0560 & 0.0880 & 0.1759 & 0.2225 & 0.1650 & 0.0670 & 0.0980 \\
 & 0.8 & 0.1985 & 0.2169 & 0.1738 & 0.0710 & 0.1040 & 0.1849 & \textbf{0.2487} & 0.1752 & \textbf{0.0860} & 0.0950 \\
 & 0.9 & 0.2220 & \textbf{0.2224} & 0.1860 & 0.0680 & 0.1230 & 0.1974 & 0.2477 & 0.1826 & \textbf{0.0860} & 0.1100 \\
\midrule
\multirow{4}{*}{DPP}
 & 0.6 & 0.1739 & 0.0749 & 0.0842 & 0.0160 & 0.0360 & 0.1680 & 0.0838 & 0.0886 & 0.0130 & 0.0390 \\
 & 0.7 & 0.1672 & 0.1131 & 0.1101 & 0.0210 & 0.0540 & 0.1486 & 0.1236 & 0.1059 & 0.0250 & 0.0470 \\
 & 0.8 & 0.1796 & 0.1628 & 0.1417 & 0.0400 & 0.0670 & 0.1544 & 0.1846 & 0.1365 & 0.0440 & 0.0630 \\
 & 0.9 & 0.1942 & 0.2190 & 0.1726 & 0.0670 & 0.1020 & 0.1768 & 0.2318 & 0.1652 & 0.0730 & 0.0900 \\
\midrule
\multirow{4}{*}{FW (ours)}
 & 0.6 & 0.2208 & 0.1750 & 0.1637 & 0.0480 & 0.0990 & 0.1990 & 0.2145 & 0.1730 & 0.0620 & 0.0990 \\
 & 0.7 & 0.2246 & 0.1970 & 0.1768 & 0.0620 & 0.1180 & 0.1988 & 0.2304 & 0.1769 & 0.0660 & 0.1050 \\
 & 0.8 & \textbf{0.2330} & 0.2198 & 0.1900 & \textbf{0.0810} & \textbf{0.1290} & 0.1952 & 0.2435 & 0.1822 & 0.0820 & 0.1040 \\
 & 0.9 & 0.2266 & 0.2215 & \textbf{0.1906} & 0.0770 & 0.1240 & 0.2013 & 0.2453 & \textbf{0.1865} & 0.0840 & 0.1100 \\
\bottomrule
\end{tabular}
}
\caption{\textbf{End-to-end generation results on QAMPARI.} The downstream LLM generation quality across $k \in \{25, 50\}$ and $\theta \in \{ 0.6, 0.7, 0.8, 0.9 \}$. P, R, and F1 denote instance-level precision, recall, and F1 score, respectively. R$\geq$0.8 and F1$\geq$0.5 indicate the proportion of queries achieving a recall of at least $0.8$ or an F1 score of at least $0.5$. Notably, our algorithm with $\theta = 0.9$ consistently achieves the highest overall F1 score across both subset sizes.}
\label{tab:generation_qampari}
\end{table*}

In the main text, we specifically focus on the QAMPARI dataset. Because QAMPARI requires generating comprehensive lists of distinct entities, algorithms are measured via instance-level precision, recall, and F1 score. Table \ref{tab:generation_qampari} reveals the following insights:

\noindent $\bullet$ Excessively promoting semantic diversity ($\theta \leq 0.7$ for this dataset) degrades the generation performance. Because of Wikipedia's low redundancy, such strict diversity promotion displaces grounding facts, crippling retrieval recall and the downstream generation. We also find that introducing a modest degree of diversity generally yields modest but consistent improvements over standard top-$k$ retrieval. Notably, our Frank--Wolfe algorithm consistently achieves the highest overall F1 score.

\noindent $\bullet$ When expanding the passage retrieval size from $k = 25$ to $k = 50$, we observe a trade-off: while recall increases, precision drops noticeably. This indicates an information overload for the reader LLM, which introduces noise and causes the overall F1 score to decrease. Consequently, we omit generation experiments for $k = 100$.

Finally, generation results for the ASQA dataset exhibit similar trends and are deferred to Appendix \ref{appx5}.

\section{Discussion}

\textbf{Related Work.} Although CCBQP in general is NP-hard and difficult to approximate, instances with an intrinsic low-rank structure can sometimes be solved exactly in polynomial time \citep{asteris2014sparse,papailiopoulos2014finding}. While our formulation in \eqref{p2} with $\lambda = 0$ exhibits a similar low-rank structure, the presence of the linear term prevents these exact solvers from being trivially adapted. More importantly, even if a theoretical extension were possible, an $O(n^{d+1})$ time complexity is fundamentally prohibitive for RAG applications, as the embedding dimension $d$ typically exceeds $768$.

Beyond classical optimization frameworks, recent studies have explored various strategies to balance relevance and diversity. For instance, \citet{gao2024vector} proposed a parameter-free algorithm; however, without the trade-off parameter, this approach collapses the relevance-diversity Pareto frontier into a single point. \citet{nguyen2025muss} proposed MUSS, a greedy algorithm, to improve scalability via multilevel subset selection. A more recent work by \citet{khan2026df} demonstrates that RAG actually requires dynamically optimizing the trade-off parameter, proving that parameter-free approaches are insufficient.

\textbf{Limitations.} While this paper provides a principled and scalable framework for diversity-aware retrieval, our study has the following limitations.

First, our empirical evaluations are conducted on datasets grounded in the Wikipedia corpus. Wikipedia's low semantic redundancy limits the observable benefits of diversity-aware retrieval. Evaluating our framework on highly redundant corpora remains an important future direction to fully realize the benefits of diversity-aware retrieval.

Second, due to computational constraints, our end-to-end generation experiments were conducted using a single 8B-parameter LLM. While this model scale is widely used for RAG evaluation, larger models often exhibit different behaviors. Evaluating our framework on larger models warrants further exploration.

\section{Conclusion}

In this paper, we introduced a principled and highly scalable framework for diversity-aware retrieval in RAG. By formulating the relevance-diversity trade-off as a cardinality-constrained binary quadratic program (CCBQP), we addressed the fundamental theoretical and computational limitations of existing methods. We proposed a tight continuous relaxation of the CCBQP objective and developed a specialized Frank--Wolfe algorithm equipped with an exact line search. Our theoretical analyses confirmed the benign optimization landscape and the local exact convergence of the proposed algorithm. 

Extensive empirical evaluations demonstrated that our approach consistently dominates the relevance-diversity Pareto frontier. More importantly, by achieving a highly favorable sub-linear scaling with respect to the retrieval size $k$, our method provides significant speedups over MMR and DPP, making it well-suited for modern RAG pipelines that demand massive context windows. End-to-end generation experiments further validated the robustness and utility of our algorithm under various hyperparameter settings. Overall, this work provides a scalable and principled approach to diversity-aware retrieval, offering a robust and efficient foundation for RAG applications.

\bibliography{colm2026_conference}
\bibliographystyle{colm2026_conference}

\appendix
\section{Proof of Theorem \ref{thm1}} \label{appx1}

\begin{proof}
    Let $\boldsymbol{W} = \boldsymbol{EE}^{\mathsf{T}}$. Since each passage embedding $\boldsymbol{e}_i$ is $\ell^{2}$-normalized, the diagonal elements of the similarity matrix are $w_{ii} = \boldsymbol{e}_i^{\mathsf{T}} \boldsymbol{e}_i = 1$ for all $i \in [n]$. Furthermore, the off-diagonal elements are the cosine similarities, meaning $w_{ij} \in [-1, 1]$ for all $i \neq j$.

    Suppose that $\boldsymbol{x}$ is a non-integral feasible point of \eqref{p3}. Let $\mathcal{M}(\boldsymbol{x}) = \{ i \in [n] \mid 0 < x_i < 1 \}$ denote the set of indices corresponding to the non-integral entries of $\boldsymbol{x}$. Because the constraint requires the sum $\boldsymbol{1}^{\mathsf{T}} \boldsymbol{x} = k$ to be an integer, the cardinality of $\mathcal{M}(\boldsymbol{x})$ must be either $0$ or strictly greater than $1$.
    
    For any non-integral feasible $\boldsymbol{x}$, we can always find two distinct indices $i, j \in \mathcal{M}(\boldsymbol{x})$ such that the partial derivatives satisfy $\frac{\partial f(\boldsymbol{x})}{\partial x_i} \geq \frac{\partial f(\boldsymbol{x})}{\partial x_j}$. Let $\boldsymbol{d} = \boldsymbol{e}_i - \boldsymbol{e}_j$, where $\boldsymbol{e}_i$ is the $i$-th vector of the canonical basis for $\mathbb{R}^n$, and let $\delta = \min\{1 - x_i, x_j\}$. Since $i, j \in \mathcal{M}(\boldsymbol{x})$, we have $\delta > 0$. We define a new feasible point $\hat{\boldsymbol{x}} = \boldsymbol{x} + \delta \boldsymbol{d}$.
    
    Next, we evaluate the difference between the objective function values $f(\hat{\boldsymbol{x}})$ and $f(\boldsymbol{x})$. Using the exact Taylor expansion for the quadratic function $f$, we have:
    \begin{equation}
        f(\hat{\boldsymbol{x}}) - f(\boldsymbol{x}) = \delta \nabla f(\boldsymbol{x})^{\mathsf{T}} \boldsymbol{d} + \delta^2 \boldsymbol{d}^{\mathsf{T}} \boldsymbol{Q} \boldsymbol{d}.
    \end{equation}
    where $\boldsymbol{Q} = (1 - \theta)(\lambda \boldsymbol{I} - \boldsymbol{W})$ is the quadratic coefficient matrix of $f$. 
    
    For the first-order term, our choice of $i$ and $j$ ensures that:
    \begin{equation}
        \nabla f(\boldsymbol{x})^{\mathsf{T}} \boldsymbol{d} = \frac{\partial f(\boldsymbol{x})}{\partial x_i} - \frac{\partial f(\boldsymbol{x})}{\partial x_j} \geq 0.
    \end{equation}
    
    For the second-order term, we expand $\boldsymbol{d}^{\mathsf{T}} \boldsymbol{Q} \boldsymbol{d}$:
    \begin{equation}
        \begin{aligned}
            \boldsymbol{d}^{\mathsf{T}} \boldsymbol{Q} \boldsymbol{d} &= (1 - \theta) (\boldsymbol{e}_i - \boldsymbol{e}_j)^{\mathsf{T}} (\lambda \boldsymbol{I} - \boldsymbol{W}) (\boldsymbol{e}_i - \boldsymbol{e}_j) \\
            &= (1 - \theta) [(\lambda - w_{ii}) + (\lambda - w_{jj}) - 2(0 - w_{ij})].
        \end{aligned}
    \end{equation}
    
    Substituting $w_{ii} = w_{jj} = 1$, we obtain:
    \begin{equation}
        \boldsymbol{d}^{\mathsf{T}} \boldsymbol{Q} \boldsymbol{d} = (1 - \theta) [2\lambda - 2 + 2w_{ij}] = 2(1 - \theta)(\lambda - 1 + w_{ij}).
    \end{equation}
    
    Combining these terms, the difference between the objective function values is:
    \begin{equation}
        f(\hat{\boldsymbol{x}}) - f(\boldsymbol{x}) \geq 2\delta^2 (1 - \theta)(\lambda - 1 + w_{ij}).
    \end{equation}
    
    When $\lambda \geq 2$, we have $\lambda - 1 \geq 1$. Since $w_{ij}$ represents the cosine similarity, we know $w_{ij} \geq -1$, which implies $\lambda - 1 + w_{ij} \geq 1 - 1 = 0$. Given that $\theta \in [0, 1]$, it follows that $(1 - \theta) \geq 0$. Therefore:
    \begin{equation}
        f(\hat{\boldsymbol{x}}) - f(\boldsymbol{x}) \geq 0.
    \end{equation}
    
    Hence, the objective function value $f(\hat{\boldsymbol{x}})$ is greater than or equal to $f(\boldsymbol{x})$. By the definition of $\delta$, moving from $\boldsymbol{x}$ to $\hat{\boldsymbol{x}}$ forces either $x_i$ to $1$ or $x_j$ to $0$, meaning the cardinality of $\mathcal{M}(\hat{\boldsymbol{x}})$ is strictly smaller than the cardinality of $\mathcal{M}(\boldsymbol{x})$. By repeating this update until the cardinality of $\mathcal{M}(\boldsymbol{x})$ reduces to $0$, we obtain an integral feasible point $\boldsymbol{x}^{\ast}$ of \eqref{p3} such that $f(\boldsymbol{x}^{\ast}) \geq f(\boldsymbol{x})$. This implies that an integral optimal solution always exists, proving the relaxation is tight.
\end{proof}

\section{Proof of Theorem \ref{thm2}} \label{appx2}

\begin{proof}
    To characterize the stationary points of the relaxed problem \eqref{p3}, we first partition the index set $[n]$ for any feasible point $\boldsymbol{x}$ into three disjoint sets: $\mathcal{S}_{0}(\boldsymbol{x}) = \{i \mid x_{i} = 0\}$, $\mathcal{S}_{f}(\boldsymbol{x}) = \{i \mid 0 < x_{i} < 1\}$, and $\mathcal{S}_{1}(\boldsymbol{x}) = \{i \mid x_{i} = 1\}$. For brevity, we omit the argument $\boldsymbol{x}$ when the context is clear. We also let $\boldsymbol{W} = \boldsymbol{EE}^{\mathsf{T}}$.

    The Lagrangian function of \eqref{p3} can be expressed as:
    \begin{equation}
        L(\boldsymbol{x}, \mu, \boldsymbol{\alpha}, \boldsymbol{\beta}) = f(\boldsymbol{x}) - \mu\left( \sum_{i \in [n]} x_{i} - k \right) + \sum_{i\in[n]}\alpha_i x_i - \sum_{i\in[n]}\beta_i(x_i - 1).
    \end{equation}
    
    The Karush-Kuhn-Tucker (KKT) conditions are:
    \begin{equation}
        \begin{cases}
            \nabla_i f(\boldsymbol{x}) = \mu - \alpha_i + \beta_i, \\
            \alpha_i \geq 0, \quad \beta_i \geq 0, \\
            \alpha_i x_i = 0, \quad \beta_i(x_i - 1) = 0,
        \end{cases}
    \end{equation}
    for every $i \in [n]$, which is equivalent to:
    \begin{equation} \label{eq:kkt_simplified}
        \nabla_{i} f(\boldsymbol{x}) \leq \mu \text{ for } i \in \mathcal{S}_{0}, \quad 
        \nabla_{i} f(\boldsymbol{x}) \geq \mu \text{ for } i \in \mathcal{S}_{1}, \quad 
        \nabla_{i} f(\boldsymbol{x}) = \mu \text{ for } i \in \mathcal{S}_{f}.
    \end{equation}
    
    Next, we evaluate the exact second-order Taylor expansion at a stationary point $\boldsymbol{x}$. For any feasible swap direction $\boldsymbol{d} = \boldsymbol{e}_{i} - \boldsymbol{e}_{j}$ and step size $\delta > 0$, we have:
    \begin{equation} \label{eq:taylor}
        f(\boldsymbol{x} + \delta \boldsymbol{d}) - f(\boldsymbol{x}) = \delta \boldsymbol{d}^{\mathsf{T}} \nabla f(\boldsymbol{x}) + \frac{1}{2} \delta^{2} \boldsymbol{d}^{\mathsf{T}} \nabla^{2} f(\boldsymbol{x}) \boldsymbol{d}.
    \end{equation}
    The first-order term is $\delta (\nabla_{i} f(\boldsymbol{x}) - \nabla_{j} f(\boldsymbol{x}))$. The Hessian is $\nabla^{2} f(\boldsymbol{x}) = 2(1 - \theta)(\lambda \boldsymbol{I} - \boldsymbol{W})$. Since the embeddings are $\ell^{2}$-normalized, the second-order term can be evaluated as:
    \begin{equation} \label{eq:hessian_term}
        \begin{aligned}
            \frac{1}{2} \delta^{2} \boldsymbol{d}^{\mathsf{T}} \nabla^{2} f(\boldsymbol{x}) \boldsymbol{d} &= \delta^{2} (1 - \theta) (\boldsymbol{e}_{i} - \boldsymbol{e}_{j})^{\mathsf{T}} (\lambda \boldsymbol{I} - \boldsymbol{W}) (\boldsymbol{e}_{i} - \boldsymbol{e}_{j}) \\
            &= \delta^{2} (1 - \theta) (2\lambda - w_{ii} - w_{jj} + 2w_{ij}) \\
            &= 2\delta^{2} (1 - \theta) (\lambda - 1 + w_{ij}).
        \end{aligned}
    \end{equation}
    Under Assumption \ref{assump1}, $w_{ij} > -1$. Given $\lambda \geq 2$ and $\theta \in [0, 1)$, we have that the second-order term \eqref{eq:hessian_term} is strictly positive.
    
    Now, we categorize all stationary points into two cases:
    
    \textbf{Case 1: Non-integral stationary points.} \\
    Suppose that $\boldsymbol{x}$ is a non-integral stationary point. Since $\sum_{i \in [n]} x_{i} = k$, the set $\mathcal{S}_f$ must contain at least two distinct indices $i$ and $j$. Let $\boldsymbol{d} = \boldsymbol{e}_{i} - \boldsymbol{e}_{j}$ and $\delta = \min\{x_j, 1 - x_i\} > 0$. By the KKT conditions \eqref{eq:kkt_simplified}, $\nabla_{i} f(\boldsymbol{x}) = \nabla_{j} f(\boldsymbol{x})$, which makes the first-order term in \eqref{eq:taylor} zero. Consequently, $f(\boldsymbol{x} + \delta \boldsymbol{d}) - f(\boldsymbol{x}) = 2\delta^2 (1 - \theta) (\lambda - 1 + w_{ij}) > 0$. Therefore, $\boldsymbol{d}$ is a feasible strict ascent direction, proving that all non-integral stationary points are strict saddle points.
    
    \textbf{Case 2: Integral stationary points.} \\
    Suppose that $\boldsymbol{x}^{\ast}$ is an integral stationary point. By the KKT conditions \eqref{eq:kkt_simplified}, we have $\max_{l \in \mathcal{S}_{0}} \nabla_{l} f(\boldsymbol{x}^{\ast}) \leq \min_{m \in \mathcal{S}_{1}} \nabla_{m} f(\boldsymbol{x}^{\ast})$. 
    For any non-zero feasible direction $\boldsymbol{d}$, since $\boldsymbol{x}^{\ast} \in \{0, 1\}^n$, we must have $d_{l} \geq 0$ for $l \in \mathcal{S}_{0}$ and $d_{m} \leq 0$ for $m \in \mathcal{S}_{1}$. The constraint $\boldsymbol{1}^{\mathsf{T}} \boldsymbol{d} = 0$ ensures that $\sum_{l \in \mathcal{S}_{0}} d_{l} = -\sum_{m \in \mathcal{S}_{1}} d_{m} > 0$. The first-order directional derivative can thus be bounded by:
    \begin{equation} \label{eq:first_order_bound}
        \begin{aligned}
            \boldsymbol{d}^{\mathsf{T}} \nabla f(\boldsymbol{x}^{\ast}) &= \sum_{l \in \mathcal{S}_{0}} d_{l} \nabla_{l} f(\boldsymbol{x}^{\ast}) + \sum_{m \in \mathcal{S}_{1}} d_{m} \nabla_{m} f(\boldsymbol{x}^{\ast}) \\
            &\leq \left(\sum_{l \in \mathcal{S}_{0}} d_{l}\right) \left( \max_{l \in \mathcal{S}_{0}} \nabla_{l} f(\boldsymbol{x}^{\ast}) - \min_{m \in \mathcal{S}_{1}} \nabla_{m} f(\boldsymbol{x}^{\ast}) \right).
        \end{aligned}
    \end{equation}
    If $\max_{l \in \mathcal{S}_{0}} \nabla_{l} f(\boldsymbol{x}^{\ast}) < \min_{m \in \mathcal{S}_{1}} \nabla_{m} f(\boldsymbol{x}^{\ast})$, then $\boldsymbol{d}^{\mathsf{T}} \nabla f(\boldsymbol{x}^{\ast}) < 0$ for all non-zero feasible directions. This satisfies the first-order sufficient condition for local optimality, making $\boldsymbol{x}^{\ast}$ a local maximizer regardless of the second-order curvature.
    
    Therefore, if an integral stationary point $\boldsymbol{x}^{\ast}$ is \emph{not} a local maximizer, we must have $\max_{l \in \mathcal{S}_{0}} \nabla_{l} f(\boldsymbol{x}^{\ast}) = \min_{m \in \mathcal{S}_{1}} \nabla_{m} f(\boldsymbol{x}^{\ast})$. Consequently, there must exist at least one pair of indices $i \in \mathcal{S}_{0}$ and $j \in \mathcal{S}_{1}$ such that $\nabla_{i} f(\boldsymbol{x}^{\ast}) = \nabla_{j} f(\boldsymbol{x}^{\ast})$. 
    
    Let $\boldsymbol{d} = \boldsymbol{e}_{i} - \boldsymbol{e}_{j}$ and $\delta = 1$. The point $\boldsymbol{x}^{\ast} + \boldsymbol{d}$ remains feasible. Because $\nabla_{i} f(\boldsymbol{x}^{\ast}) = \nabla_{j} f(\boldsymbol{x}^{\ast})$, the first-order term in the Taylor expansion \eqref{eq:taylor} vanishes. The exact second-order Taylor expansion yields $f(\boldsymbol{x}^{\ast} + \boldsymbol{d}) - f(\boldsymbol{x}^{\ast}) = 2(1 - \theta) (\lambda - 1 + w_{ij}) > 0$. This explicitly demonstrates strictly positive curvature along the feasible direction $\boldsymbol{d}$, proving that any integral stationary point that is not a local maximizer is a strict saddle point.
    
    \textbf{Conclusion:} Every stationary point is either an integral local maximizer or a strict saddle point with explicit positive curvature. This establishes the strict dichotomy.
\end{proof}

\section{Proof of Theorem \ref{thm3}} \label{appx3}

\begin{proof}
    Observe that the objective function $f_{\lambda}(\boldsymbol{x})$ can be decomposed as:
    \begin{equation}
        f_{\lambda}(\boldsymbol{x}) = \theta (k - 1) \boldsymbol{c}^{\mathsf{T}} \boldsymbol{x} - (1 - \theta) \boldsymbol{x}^{\mathsf{T}} \boldsymbol{Wx} + (1 - \theta) \lambda \Vert \boldsymbol{x} \Vert_{2}^{2}.
    \end{equation}
    Therefore, for any $\lambda_{2} > \lambda_{1} \geq 2$, the difference between the two objective functions is:
    \begin{equation} \label{eq:obj_diff}
        f_{\lambda_2}(\boldsymbol{x}) = f_{\lambda_1}(\boldsymbol{x}) + (1 - \theta)(\lambda_2 - \lambda_1) \Vert \boldsymbol{x} \Vert_{2}^{2}.
    \end{equation}

    Let $\boldsymbol{x}^{\ast}$ be a local maximizer of $f_{\lambda_{1}}(\boldsymbol{x})$. Under Assumption \ref{assump1}, by Theorem \ref{thm2}, $\boldsymbol{x}$ is integral. By the definition of local optimality, there exists a neighborhood $\mathcal{N}$ around $\boldsymbol{x}^{\ast}$ such that for any feasible point $\boldsymbol{x} \in \mathcal{N}$, we have $f_{\lambda_{1}}(\boldsymbol{x}) \leq f_{\lambda_{1}}(\boldsymbol{x}^{\ast})$ and $\Vert \boldsymbol{x} \Vert_{2}^{2} \leq \Vert \boldsymbol{x}^{\ast} \Vert_{2}^{2}$.

    Combining these with \eqref{eq:obj_diff}, for any feasible $\boldsymbol{x} \in \mathcal{N}$, we evaluate the objective under $\lambda_{2}$:
    \begin{equation}
        \begin{aligned}
        f_{\lambda_{2}}(\boldsymbol{x}) &= f_{\lambda_{1}}(\boldsymbol{x}) + (1 - \theta)(\lambda_{2} - \lambda_{1}) \Vert \boldsymbol{x} \Vert_{2}^{2} \\
        &\leq f_{\lambda_1}(\boldsymbol{x}^{\ast}) + (1 - \theta)(\lambda_2 - \lambda_1) \Vert \boldsymbol{x}^{\ast} \Vert_{2}^{2} \\
        &= f_{\lambda_2}(\boldsymbol{x}^{\ast}).
        \end{aligned}
    \end{equation}
    Thus, the local optimality under $\lambda_{1}$ is preserved  under $\lambda_{2}$.
\end{proof}

\section{Proof of Theorem \ref{thm4}} \label{appx4}

\begin{proof}
    Let $\mathcal{S}_{0}(\boldsymbol{x}) = \{i \mid x_{i} = 0\}$ and $\mathcal{S}_{1}(\boldsymbol{x}) = \{i \mid x_{i} = 1\}$. By the local optimality condition established in Appendix \ref{appx2}, we have $\max_{l \in \mathcal{S}_{0}(\boldsymbol{x}^{\ast})} \nabla_{l} f(\boldsymbol{x}^{\ast}) < \min_{m \in \mathcal{S}_{1}(\boldsymbol{x}^{\ast})} \nabla_{m} f(\boldsymbol{x}^{\ast})$. Because the objective function $f(\boldsymbol{x})$ is $L$-smooth, there exists a sufficiently small continuous neighborhood $\mathcal{N}(\boldsymbol{x}^{\ast})$ around $\boldsymbol{x}^{\ast}$ such that for any feasible point $\boldsymbol{x} \in \mathcal{N}(\boldsymbol{x}^{\ast})$, we have $\max_{l \in \mathcal{S}_{0}(\boldsymbol{x}^{\ast})} \nabla_{l} f(\boldsymbol{x}) < \min_{m \in \mathcal{S}_{1}(\boldsymbol{x}^{\ast})} \nabla_{m} f(\boldsymbol{x})$.

    Consequently, when the non-stationary point $\boldsymbol{x}^{(t)}$ is within $\mathcal{N}(\boldsymbol{x}^{\ast})$, the linear maximization oracle in Algorithm \ref{alg:fw} will uniquely identify the local maximizer $\boldsymbol{x}^{\ast}$:
    \begin{equation}
        \boldsymbol{s}^{(t)} = \operatorname{top-k}(\nabla f(\boldsymbol{x}^{(t)})) = \boldsymbol{x}^{\ast}.
    \end{equation}
    The update direction is thus $\boldsymbol{d}^{(t)} = \boldsymbol{x}^{\ast} - \boldsymbol{x}^{(t)}$.

    Next, we prove that the exact line search along $\boldsymbol{d}^{(t)}$ yields $\gamma^{(t)} = 1$. Let $g(\gamma) = f(\boldsymbol{x}^{(t)} + \gamma \boldsymbol{d}^{(t)})$. We evaluate the derivative of $g(\gamma)$ at $\gamma = 1$:
    \begin{equation}
        g'(1) = \langle \nabla f(\boldsymbol{x}^{(t)} + 1 \cdot \boldsymbol{d}^{(t)}), \boldsymbol{d}^{(t)} \rangle = \langle \nabla f(\boldsymbol{x}^{\ast}), \boldsymbol{x}^{\ast} - \boldsymbol{x}^{(t)} \rangle.
    \end{equation}
    Expanding this inner product over the sets $\mathcal{S}_{1}(\boldsymbol{x}^{\ast})$ and $\mathcal{S}_{0}(\boldsymbol{x}^{\ast})$, we obtain:
    \begin{equation}
        g'(1) = \sum_{m \in \mathcal{S}_{1}(\boldsymbol{x}^{\ast})} \nabla_{m} f(\boldsymbol{x}^{\ast}) (1 - x^{(t)}_{m}) + \sum_{l \in \mathcal{S}_{0}(\boldsymbol{x}^{\ast})} \nabla_{l} f(\boldsymbol{x}^{\ast}) (0 - x^{(t)}_{l}).
    \end{equation}
    Let $\sum_{m \in \mathcal{S}_{1}(\boldsymbol{x}^{\ast})} (1 - x^{(t)}_{m}) = \sum_{l \in \mathcal{S}_{0}(\boldsymbol{x}^{\ast})} x^{(t)}_{l} = \Delta_{\boldsymbol{x}} > 0$. We have the lower bound of $g'(1)$:
    \begin{equation}
        \begin{aligned}
            g'(1) &\geq \left( \min_{m \in \mathcal{S}_{1}(\boldsymbol{x}^{\ast})} \nabla_{m} f(\boldsymbol{x}^{\ast}) \right) \sum_{m \in \mathcal{S}_{1}(\boldsymbol{x}^{\ast})} (1 - x^{(t)}_{m}) - \left( \max_{l \in \mathcal{S}_{0}(\boldsymbol{x}^{\ast})} \nabla_{l} f(\boldsymbol{x}^{\ast}) \right) \sum_{l \in \mathcal{S}_{0}(\boldsymbol{x}^{\ast})} x^{(t)}_{l} \\
        &= \left( \min_{m \in \mathcal{S}_{1}(\boldsymbol{x}^{\ast})} \nabla_{m} f(\boldsymbol{x}^{\ast}) - \max_{l \in \mathcal{S}_{0}(\boldsymbol{x}^{\ast})} \nabla_{l} f(\boldsymbol{x}^{\ast}) \right) \Delta_{\boldsymbol{x}} > 0.
        \end{aligned}
    \end{equation}
 
    Let $C = \boldsymbol{d}^{(t)\mathsf{T}} \nabla^{2} f(\boldsymbol{x}^{(t)}) \boldsymbol{d}^{(t)}$ be the curvature. There are two cases for the exact line search:
    \begin{enumerate}
        \item If $C \geq 0$, $g(\gamma)$ is strictly monotonically increasing on $[0, 1]$. Thus, $g(\gamma)$ is maximized within the feasible interval $[0, 1]$ if and only if $\gamma = 1$.
        \item If $C < 0$, the parabola $g(\gamma)$ is strictly concave and the axis of symmetry for the parabola is $-\frac{\langle \nabla f(\boldsymbol{x}^{(t)}), \boldsymbol{d}^{(t)} \rangle}{C} > 0$. Since $g'(1) > 0$, we have $-\frac{\langle \nabla f(\boldsymbol{x}^{(t)}), \boldsymbol{d}^{(t)} \rangle}{C} > 1$. Thus, $g(\gamma)$ is maximized within the feasible interval $[0, 1]$ if and only if $\gamma = 1$.
    \end{enumerate}

    In both cases, the exact line search returns $\gamma^{(t)} = 1$. Hence, in the next iteration, Algorithm~\ref{alg:fw} will exactly converge to:
    \begin{equation}
        \boldsymbol{x}^{(t+1)} = \boldsymbol{x}^{(t)} + 1 \cdot (\boldsymbol{x}^{\ast} - \boldsymbol{x}^{(t)}) = \boldsymbol{x}^{\ast}.
    \end{equation}
\end{proof}

\section{Generation results for ASQA} \label{appx5}

\begin{table}[t]
\centering
\resizebox{0.6\columnwidth}{!}{
\begin{tabular}{@{}ll | cc | cc@{}}
\toprule
\multirow{2}{*}{\textbf{Method}} & \multirow{2}{*}{$\boldsymbol{\theta}$} & \multicolumn{2}{c|}{$\boldsymbol{k = 25}$} & \multicolumn{2}{c}{$\boldsymbol{k = 50}$} \\
\cmidrule(lr){3-4} \cmidrule(l){5-6}
& & Recall & Recall=1.0 & Recall & Recall=1.0 \\
\midrule
Top-$k$ & --- & 0.4819 & 0.2099 & 0.4746 & 0.2120 \\
\midrule
\multirow{4}{*}{MMR}
 & 0.6 & 0.4136 & 0.1646 & 0.4239 & 0.1614 \\
 & 0.7 & 0.4507 & 0.1867 & 0.4591 & 0.1888 \\
 & 0.8 & 0.4795 & 0.2173 & \textbf{0.4789} & \textbf{0.2131} \\
 & 0.9 & \textbf{0.4947} & \textbf{0.2226} & 0.4722 & 0.2120 \\
\midrule
\multirow{4}{*}{DPP}
 & 0.6 & 0.3404 & 0.1129 & 0.3357 & 0.1086 \\
 & 0.7 & 0.3826 & 0.1414 & 0.3692 & 0.1319 \\
 & 0.8 & 0.4275 & 0.1719 & 0.4308 & 0.1656 \\
 & 0.9 & 0.4707 & 0.2152 & 0.4677 & 0.2015 \\
\midrule
\multirow{4}{*}{FW (ours)}
 & 0.6 & 0.4655 & 0.2057 & 0.4561 & 0.2025 \\
 & 0.7 & 0.4661 & 0.1994 & 0.4665 & 0.1994 \\
 & 0.8 & 0.4830 & 0.2141 & 0.4756 & 0.2057 \\
 & 0.9 & 0.4757 & 0.2078 & 0.4721 & 0.2089 \\
\bottomrule
\end{tabular}
}
\caption{\textbf{End-to-end generation results on ASQA.} The downstream LLM generation quality across $k \in \{25, 50\}$ and $\theta \in \{ 0.6, 0.7, 0.8, 0.9 \}$. Recall=1.0 indicates the proportion of queries achieving a recall of 1.0.}
\label{tab:generation_asqa}
\end{table}

Unlike QAMPARI, which requires generating a structured list of discrete entities, ASQA requires generating a paragraph. Consequently, algorithms are measured via instance-level recall, as calculating generation precision or F1 scores is difficult for this dataset.

The empirical results in Table \ref{tab:generation_asqa} reveal two key observations:

\noindent $\bullet$ Introducing semantic diversity allows both MMR and our Frank--Wolfe algorithm to yield recall scores that surpass the standard top-$k$ baseline. Specifically, Frank--Wolfe with $\theta=0.8$ consistently maintains a slight recall advantage over top-$k$.

\noindent $\bullet$ While MMR achieves a higher peak recall and recall=1.0 compared to our Frank--Wolfe algorithm, our algorithm offers two advantages. First, Frank--Wolfe is significantly faster than MMR. Second, Frank--Wolfe exhibits superior hyperparameter robustness, as MMR and DPP suffer severe performance degradation at moderate $\theta$ values. This robustness can be attributed to our algorithm's substantial advantage on the retrieval Pareto frontier at moderate $\theta$ values. 

\end{document}